\documentclass[]{article}
\usepackage{proceed2e}

\usepackage{times}
\usepackage{natbib}
\usepackage{xspace}
\usepackage{amsmath}
\usepackage{amssymb}
\usepackage{amsthm}
\usepackage{algorithm}
\usepackage{algorithmic}
\usepackage[hidelinks]{hyperref}
\usepackage{graphicx}
\usepackage{calc}
\usepackage{subcaption}

\newcommand{\sumT}{\sum_{t = 1}^T}
\newcommand{\sumi}{\sum_{i = 1}^N}
\newcommand{\sumk}{\sum_{k = 1}^{N-2}}
\newcommand{\sumTk}{\sum_{t \in T_k}}
\newcommand{\bpa}[1]{\bigl(#1\bigr)}

\newcommand{\exph}{\texttt{Exp3}\xspace}

\newcommand{\alg}{\texttt{Exp3-Res}\xspace}
\newcommand{\expr}{\texttt{Exp3-R}\xspace}
\newcommand{\expset}{\texttt{Exp3-SET}\xspace}

\newcommand{\exporacle}{\texttt{Oracle}\xspace}
\newcommand{\hedge}{\texttt{Hedge}\xspace}

\newcommand{\loss}{\ell}
\newcommand{\ev}[1]{\left\{#1\right\}}
\newcommand{\EE}[1]{\mathbb{E}\left[ #1\right]}
\newcommand{\pa}[1]{\left(#1\right)}
\newcommand{\F}{\mathcal{F}}
\newcommand{\PPcc}[2]{\mathbb{P}\left[\left.#1\right|#2\right]}
\newcommand{\hloss}{\wh{\loss}}
\newcommand{\wh}{\widehat}
\newcommand{\ti}{_{t,i}}
\newcommand{\EEcc}[2]{\mathbb{E}\left[\left.#1\right|#2\right]}
\newcommand{\hL}{\wh{L}}
\newcommand{\PPt}[1]{\mathbb{P}_t\left[#1\right]}
\newcommand{\EEt}[1]{\mathbb{E}_t\left[ #1\right]}
\newcommand{\PP}[1]{\mathbb{P}\left[#1\right]}
\newcommand{\hLoss}{\widehat{L}}
\newcommand{\etat}{\eta_t}

\newcommand{\wt}{\widetilde}
\newcommand{\cO}{\mathcal{O}}

\newtheorem{lemma}{Lemma}
\newtheorem{theorem}{Theorem}

\newcommand*{\MyDef}{\mathrm{\tiny def}}
\newcommand*{\eqdefU}{\ensuremath{\mathop{\overset{\MyDef}{=}}}}
\newcommand*{\eqdef}{\mathop{\overset{\MyDef}{\resizebox{\widthof{\eqdefU}}{\heightof{=}}{=}}}}

\title{Online learning with Erd\H os--R\'enyi side-observation graphs}

\author{ {\bf Tom\'a\v s Koc\'ak} \\
SequeL team \\
INRIA Lille - Nord Europe \\
\And
{\bf Gergely Neu}  \\
Universitat Pompeu Fabra \\
Barcelona, Spain \\
\And
{\bf Michal Valko}   \\
SequeL team \\
INRIA Lille - Nord Europe \\
}

\begin{document}

\maketitle


\begin{abstract} 
We consider adversarial multi-armed bandit problems where the learner is 
allowed to observe losses of a number of arms
beside the arm that it actually chose. We study the case where all non-chosen arms 
reveal their loss with an \emph{unknown} probability $r_t$, independently of each 
other and the action of the learner. Moreover, we allow $r_t$ to change in every round $t$, 
which rules out the possibility of estimating $r_t$ by a well-concentrated sample average. 
We propose an algorithm which operates under the assumption that $r_t$ is large enough to warrant
at least one side observation with high probability. We show that after $T$ rounds in a bandit 
problem with~$N$ arms, the expected regret of our algorithm is of order
$\cO\Big(\sqrt{\sumT (1/r_t) \log N }\,\,\Big)$, given that $r_t\ge \log T / (2N-2)$ for all $t$. 
All our bounds are within logarithmic factors of the best achievable performance of any algorithm that is even allowed to 
know exact values of $r_t$.
\end{abstract}

\section{INTRODUCTION}

In sequential learning, a learner is repeatedly asked 
to choose an action for which it obtains a loss and
receives a feedback from the environment \citep{cesa-bianchi2006prediction}. We typically study two feedback settings: the learner
either 
observes the losses for all the potential actions (full information) or it
observes only the loss of the action it chose. This latter feedback scheme is known  as the \emph{bandit} setting
(cf.~\citealp{auer2002nonstochastic}). In this paper, instead of 
considering these two limit cases, we study a more refined feedback 
model, known as \textit{bandit with side observations} 
\citep{mannor2011from,alon2013from,kocak2014efficient,kocak2016online}, that generalizes both of 
them. Typical examples for learning with full information and bandit feedback are sequential trading on a stock market 
(where all stock prices are fully observable after each trading period), and electronic advertising (where the learner can only observe 
the clicks on actually shown ads), respectively.
However, advertising in a social network offers a more intricate user feedback 
than captured by the basic bandit model: when proposing an item to a user in 
a social network, the advertiser can often learn about the preferences of the user's 
connections as well. Naturally, the advertiser would want to improve its 
recommendation strategy by incorporating these side observations.

Besides advertising and recommender systems, side observations can also arise 
in sensor networks, where the action of the learner amounts to probing a particular sensor. 
In this setting, each sensor can reveal readings of some other sensors that are in its range. 
When our goal is to sequentially select a sensor maximizing a property of interest, 
a good learning strategy should be able to leverage these side readings. 




In this paper, we follow the formalism of \citet{mannor2011from} who model side observations with a graph structure over the
actions: two actions mutually reveal their losses if they are connected by an edge in the graph in question.
In a realistic scenario this graph is \textit{time dependent} and \textit{unknown} to the learner (e.g., the 
advertiser or the algorithm scheduling sensor readings).
All previous algorithms for the studied setting
\citep{mannor2011from,alon2013from,kocak2014efficient,kocak2016online}
require the environment to reveal a substantial part of a graph, at least 
after the side observations have been revealed. Specifically, these algorithms
require the knowledge of the \textit{second neighborhood} (the set of neighbors of the 
neighbors) of the chosen action in order to update their internal loss estimates.
On the other hand, they are able to handle arbitrary graph structures, 
potentially chosen by an adversary and prove performance guarantees 
 using graph properties based on cliques or independence sets.

The main contribution of our work is a learning algorithm that, unlike previous solutions, does 
\textit{not require the knowledge of the exact graph} underlying the observations, beyond knowing from 
which nodes the side observations came from. Relaxing this assumption, however, has to come with a price: As the very recent results of 
\citet*{CHK16} show, achieving nontrivial advantages from side observations may be impossible without perfectly known side-observation 
graphs when an adversary is allowed to pick \emph{both} the losses and the side-observation graphs. On the positive side, 
\citeauthor{CHK16}\ offer efficient algorithms achieving strong improvements over the standard regret guarantees under the assumption that 
the losses are generated in an i.i.d.~fashion and the graphs may be generated adversarially. Complementing these results,  we consider 
the case of adversarial losses and make the assumption that the side-observation graph in round $t$ is generated from an \emph{Erd\H 
os--R\'enyi} model with an \emph{unknown} and \emph{time-dependent} parameter~$r_t$. The main challenge for the learner is then the 
necessity to exploit the side observations despite not knowing the sequence $(r_t)$. It is easy to see that this model can be equivalently 
understood as each non-chosen arm revealing its loss with probability~$r_t$, independently of all other observations. That said, we still 
find it useful 
to think of the side observations as being generated from an Erd\H{o}s--R\'enyi model, as it allows direct comparisons with the related 
literature. In particular, the case of learning with Erd\H{o}s--R\'enyi side-observation graphs was considered before by 
\citet{alon2013from}: Given \emph{full access} to the underlying graph structure, their algorithm \expset can be shown to guarantee a 
regret bound of $\cO\bpa{\sqrt{\sum_t(1/r_t) (1-(1-r_t)^N) \log N}}$. 
While the assumption of having full access to the graph can be dropped relatively easily in this particular case, exact knowledge of $r_t$ 
seems to be crucial for constructing reliable loss estimates and use them to guide the choice of action in each round. 

It turns out that the problem of estimating $r_t$ while striving to perform efficiently is in fact a major difficulty in our setting.
Indeed, as we allow $r_t$ to change arbitrarily between each round, we cannot rely on any past observations to construct well-concentrated 
estimates of these parameters. That is, the main challenge is estimating $r_t$ from only a handful of samples.
The core technical tool underlying our approach is a direct estimation procedure for the losses that does not estimate $r_t$ explicitly.

Armed with this estimation procedure, we propose a learning algorithm called \alg that guarantees a regret of $\cO(\sqrt{\sum_t (1/r_t)\log N})$, 
provided that $r_t \ge \log T/(2N-2)$ holds for all rounds $t$. This assumption essentially corresponds to requiring that, with high 
probability, at least $1$ side observation is produced in every round, or, in other words, the side-observation graphs encountered are all 
\emph{non-empty}. Notice that for the assumed range of $r_t$'s, our regret bound improves upon the standard regret bound 
of \exph, which is 
of $\cO(\sqrt{NT\log N})$. It is easy to see that when $r_t$ becomes smaller than $1/N$, side observations become unreliable and the bound of 
\exph cannot be improved. That is, if our assumption cannot be verified a priori, then ignoring all side observations and using the \exph 
algorithm of \citet{auer2002nonstochastic} instead can yield a better performance. On the other hand, given that our assumption holds, our bounds 
cannot be significantly improved as suggested by the lower-bound of $\Omega(\sqrt{T/r})$ proved for a static $r$ by \citealt{alon2013from}.


Many other partial-information settings have been studied in previous work. One of the simplest of these settings is 
the label-efficient prediction game considered by~\citet{cesa-bianchi2005minimizing}, where the learner can observe either losses of all 
the actions or none of them, not even the loss of the chosen action. This observation can be queried by the learner at 
most an $\varepsilon<1$ fraction of the total number of rounds, which means no losses are observed in the remaining 
rounds. An even more restricted information setting, label efficient bandit feedback was considered
by~\citet{allenberg2006hannan}, where the learner can only query the loss of the chosen action, instead of all losses 
(see also~\citealp{audibert2010regret}). Algorithms for these two settings have regret of $\widetilde 
\cO(\sqrt{T/\varepsilon)}$ and $\widetilde \cO(\sqrt{NT/\varepsilon)}$, respectively. 
While these bounds may appear very similar to ours, notice that our setting offers a more intricate (and, for some problems, more 
realistic) feedback scheme, which also turns out to be much more challenging to exploit. In another related setting, \citet{seldin2014prediction} 
consider $M$ side observations that the learner can proactively choose in each round without limitations. \citeauthor{seldin2014prediction}~deliver an 
algorithm with regret of $\widetilde \cO(\sqrt{(N/M)T)}$, also proving that choosing $M$ 
observations uniformly at random is minimax optimal; given this sampling scheme, it is not even necessary to observe the loss of the chosen 
action. Their result is comparable to ours and the result by~\citet{alon2013from} for Erd\H os--R\'enyi observation graphs with parameter 
$r=M/N$. However, \citeauthor{seldin2014prediction}~also assume that $M$ is known, which obviates the need for estimating 
$r$. We provide a more technical discussion on the related work in Section~\ref{sec:conc}.

In our paper, we assume that, just like the observation probabilities, the losses are \textit{ad\-ver\-sa\-ri\-al}, that is, they 
can change at each time step without restrictions. Learning with side observations and stochastic losses was studied 
by \citet{caron2012leveraging} and \citet{buccapatnam2014stochastic}. 
While this is an easier setting that the adversarial one, the authors assumed, in both cases, that the graphs have to be known 
in advance. Recently,  \citet{carpentier2016revealing}  studied another stochastic setting where the graph is also not known in advance, however 
their setting considers different feedback and loss structure (influence maximization) which differs from the side-observation setting.

Furthermore, \citet{alon2015online} considered a strictly more difficult setting than ours, where the   
loss of the chosen action may not be a part of the received feedback.

\section{PROBLEM DEFINITION}
We now formalize our learning problem. We consider a sequential interaction 
scheme between a learner and an
environment, where the following steps are repeated in every round 
$t=1,2,\dots,T$:
\begin{enumerate}
 \item The environment chooses $r_t\in[0,1]$ and a loss function over the arms, with $\loss_{t,i}$ 
being the loss associated with arm
$i\in[N] \eqdef \ev{1,2,\dots,N}$ at time $t$.
 \item Based on its previous observations (and possibly some 
randomness), the learner draws an arm $I_t\in
[N]$.
 \item The learner suffers loss $\loss_{t,I_t}$.
 \item For all $i\neq I_t$, $O_{t,i}$ is independently drawn from a Bernoulli 
distribution with mean $r_t$. Furthermore,
$O_{t,I_t}$ is set as $1$.
 \item For all $i\in[N]$ such that $O_{t,i}=1$, the learner observes the loss 
$\loss_{t,i}$.
\end{enumerate}
The goal of the learner is to minimize its total expected losses, or, equivalently, to minimize the
\emph{total expected regret} (or, in short, regret) defined
as
\[
 R_T = \max_{i\in[N]}\EE{\sum_{t=1}^T \pa{\loss_{t,I_t} - \loss_{t,i}}}.
\]
We will denote the interaction history between the learner and the 
environment up to the beginning of round $t$ by~$\F_{t-1}$. 
We also define $p_{t,i} = \PPcc{I_t=i}{\F_{t-1}}$.

The main challenge in our setting is leveraging side observations \emph{without 
knowing $r_t$}. Had we had access to the exact value of $r_t$, we would be able to define the 
following estimate of $\loss_{t,i}$:
\begin{equation}\label{eq:optest}
 \hloss_{t,i}^\star = \frac{O\ti\loss_{t,i}}{p_{t,i} + (1-p_{t,i})r_t} 
\end{equation}
It is easy to see that the loss estimates defined this way are unbiased in the 
sense that
$\EEcc{\hloss_{t,i}}{\F_{t-1}} = \loss_{t,i}$ for all $t$ and~$i$.
It is also straightforward to show that an appropriately tuned instance of the \exph algorithm of 
\citet{auer2002nonstochastic} fed with these loss
estimates is guaranteed to achieve a regret of $\cO(\sqrt{\sum_t(1/r_t)\log N})$ (see also 
\citealt{seldin2014prediction}). 

Then, one might consider a simple algorithm that devotes a number of observations to obtain an 
estimate $\wh{r_t}$ of $r_t$ and plug this estimate into~\eqref{eq:optest}. However, notice that since $r_t$ is allowed to change 
arbitrarily over time, we can only work with a severely limited sample budget for estimating $r_t$: only $N-1$ independent observations! 
Thus, we can obtain only very loose confidence intervals around $r_t$ which translate to even more useless confidence intervals around~$\hloss^\star_{t,i}$. 

Below, we describe a simple trick for obtaining loss estimates that have similar 
properties to the ones defined in~\eqref{eq:optest} without requiring exact knowledge or even explicit 
estimation of $r_t$. Our procedure is based on the geometric resampling method of \citet{neu2013efficient}. To get an
intuition of the method,
let us assume that we have access to the independent geometrically distributed random 
variable $G\ti^\star$ with parameter $o_{t,i} = p_{t,i} + (1-p_{t,i}) r_t$. 
Then, replacing $1/o_{t,i}$ by $G^\star_{t,i}$ in the definition of $\hloss_t^\star$ and ensuring that $G_{t,i}^\star$ is 
independent of $O_{t,i}$, we can obtain an unbiased loss estimate essentially equivalent to $\hloss_t^\star$. 

The challenge posed by this approach is that in our setting, we do not have exact sample access to the 
geometric random variable $G\ti^\star$. In the next section, we describe our algorithm that is based on replacing $G\ti^\star$ 
in the above definition by an appropriate surrogate.

\section{ALGORITHM}\label{sec:R}

\label{sec:alg}


Our algorithm is called \alg and displayed as Algorithm~\ref{alg:algorithm1}. It is based on the
\exph algorithm of \citet{auer2002nonstochastic} and crucially relies on the construction of a 
surrogate $G\ti$ of $G\ti^\star$. Throughout this section, we will assume that $r_t\geq\frac{\log T}{2N-2}$, which implies that the 
probability of having no side observations in round $t$ is of order $1/\sqrt{T}$.

The algorithm is initialized by setting $w_{1,i} = 1/N$ for all
$i\in[N]$, and then performing the updates
\begin{equation}\label{eq:update_R}
 w_{t+1,i} = \frac{1}{N}\exp\pa{-\eta_{t+1}\hL_{t,i}}
\end{equation}
after each round $t$, where $\eta_{t+1}>0$ is a parameter of the algorithm called the \emph{learning rate} in round $t$ and $\hL_{t,i}$  is 
cumulative sum of the loss estimates $\hloss_{s,i}$ up to (and including) time $t$. In round $t$, the learner draws its action $I_t$
such that $I_t=i$ holds with probability $p_{t,i} \propto w_{t,i}$. To simplify some of the notation below, we introduce the shorthand notations 
$\PPt{\cdot} = \PPcc{\cdot}{\F_{t-1}}$ and $\EEt{\cdot} = \EEcc{\cdot}{\F_{t-1}}$.

For any fixed $t,i$, we now describe an efficiently computable surrogate $G\ti$ for the geometrically distributed random variable $G\ti^\star$ 
with parameter $o\ti$ that will be used for constructing our loss estimates. In particular, our strategy will be to construct several 
independent copies $\ev{O'_{t,i}(k)}$ of $O_{t,i}$ and choosing $G_{t,i}$ as the index $k$ of the first copy with $O'_{t,i}(k)=1$. It is 
easy to see that with infinitely many copies, we could exactly recover $G_{t,i}^\star$; our actual surrogate is going to be weaker thanks to 
the smaller sample size. For clarity of notation, we will omit most  explicit references to $t$ and $i$, with the understanding that all 
calculations need to be independently executed for all pairs $t,i$. 

Let us now describe our mechanism for constructing the copies $\ev{O'(k)}$.
Since we need 
independence of $G\ti$ and $O\ti$ for our estimates, we use only side observations from actions $[N]\setminus\ev{I_t,i}$. First, let's 
define $\sigma$ as a uniform random permutation of $[N]\setminus\ev{I_t,i}$.
For all $k\in[N-2]$, we define $R(k) = O_{t,\sigma(k)}$. Note that due to the construction, $\{R(k)\}_{k = 1}^{N-2}$ are pairwise 
independent Bernoulli random variables with parameter $r_t$, independent of $O\ti$. Furthermore, knowing $p\ti$ we can define 
$P(1),\,\ldots,\,P(N-2)$ as pairwise 
independent Bernoulli random variables with parameter $p\ti$. Using $P(k)$ and $R(k)$ we define the random variable $O'(k)$ as 
\[
O'(k) = P(k) + (1-P(k))R(k)
\]
for all $k\in[N-2]$. Using independence of all previously defined random variables, it is easy to check that the 
variables $\{O'(k)\}_{k = 1}^{N-2}$ are pairwise independent Bernoulli random 
variables with expectation $o\ti = p\ti + (1-p\ti)r_t$. Now we are ready to define $G\ti$ as
\begin{equation}\label{eq:est_G}
G\ti = \min\ev{k\in[N-2]:O(k)'=1}\cup\ev{N-1}.
\end{equation}
The following lemma states some properties of $G\ti$.
\begin{lemma} \label{lem:expectationsOfG} For any value of $g$ we have
\begin{align*}
\EE{G\ti} &= \frac{1}{o\ti} - \frac{1}{o\ti}(1-o\ti)^{N-1} \\
\EE{G\ti^2} &= \frac{2-o\ti}{o\ti^2} + \frac{1}{o^2\ti}(1-o\ti)^{N-2}\times\\
&\quad \times\Big(o\ti^2+o\ti-2+2o\ti(N-2)(o\ti-1)\Big)
\end{align*}
\end{lemma}
\begin{proof}
The proof follows directly from using the definition of $G\ti$ and simplifying the sums
\begin{align*}
\EE{G\ti} &= \sumk \left[ko\ti(1-o\ti)^{k-1}\right] + \\
&\qquad\qquad+ \left(N-1\right)(1-o\ti)^{N-2},
\\
\EE{G^2\ti} &= \sumk \left[k^2o\ti(1-o\ti)^{k-1}\right] + \\
&\qquad\qquad+ \left(N-1\right)^2(1-o\ti)^{N-2}.
\end{align*}
\end{proof}
Using Lemma \ref{lem:expectationsOfG}, it is easy to see that $G\ti$ follows a truncated geometric law in the sense that
\[
 \PP{G\ti = m} = \PP{\min\ev{G\ti^\star,N-1} = m}
\]
holds for all $m\in[N-1]$. Using all this notation, we construct an estimate of $\loss_{t,i}$ as
\begin{equation}\label{eq:est_R}
 \hloss_{t,i} = G_{t,i} O_{t,i} \loss_{t,i}.
\end{equation}
\begin{algorithm}[t]
\caption{\alg}
\label{alg:algorithm1}
\begin{algorithmic}[1]
\STATE \textbf{Input:}
\STATE Set of actions $[N]$.
\STATE \textbf{Initialization:}
\STATE $\hLoss_{0,i}\gets 0$ for $i\in[N]$.
\STATE \textbf{Run:}
\FOR{$t = 1$ {\bfseries to} $T$}
\STATE $\eta_{t} \gets \sqrt{\log N\Big/\left(N^2+\sum_{s=1}^{t-1} \sumi p_{s,i}(\hloss_{s,i})^2\right)}$. 
\STATE  $w\ti\gets (1/N)\exp(-\etat\hLoss_{t-1,i})$ for $i\in[N]$.
\STATE $W_t \gets \sumi w\ti$.
\STATE $p\ti \gets w\ti/W_t$.
\STATE Choose $I_t \sim p_t = (p_{t,1},\dots,p_{t,N})$.
\STATE Receive the observation set $O_t$.
\STATE Receive the pairs $\{i,\loss\ti\}$ for all $i$ s.t.~$O\ti = 1$.
\STATE Compute $G\ti$ for all $i\in[N]$ using~\eqref{eq:est_G}.
\STATE $\hloss\ti \gets \loss\ti O\ti G\ti$ for all $i\in[N]$.
\STATE $\hLoss\ti = \hLoss_{t-1,i} + \hloss\ti$ for all $i\in[N]$.
\ENDFOR
\end{algorithmic}
\end{algorithm}
The rationale underlying this definition of $G_{t,i}$ is rather delicate. First, note that $p\ti$ is deterministic  given the history 
$\F_{t-1}$ and therefore, does not depend on
$O\ti$. Second, $O_{t,i}$ is also independent of $O_{t,j}$ for $j\not\in\{i,I_t\}$. As a result, $G\ti$ is independent of $O\ti$, and we 
can use the identity $\EEt{G_{t,i}O_{t,i}} = \EEt{G_{t,i}}\EEt{O_{t,i}}$. 
The next lemma relates the loss estimates \eqref{eq:est_R} to the true losses, relying on the observations above and the 
assumption $r_t\ge \frac{\log T}{2N-2}$. 
\begin{lemma}\label{lemma:lbias}
 Assume $r_t\ge \frac{\log T}{2N-2}$. Then, for all $t$ and $i$,
  \[
  0 \leq \loss_{t,i} - \EEt{\hloss_{t,i}} \le \frac{1}{\sqrt{T}}.
 \]
\end{lemma}
\begin{proof}
Fix an arbitrary $t$ and $i$. Using Lemma \ref{lem:expectationsOfG} along with $\EEt{O_{t,i}} = o_{t,i}$ and the independence of $G_{t,i}$ and $O_{t,i}$, we get 
 \[
 \EEt{\hloss\ti} = \EEt{G_{t,i} O_{t,i} \loss_{t,i}} = \loss_{t,i} - \loss\ti(1-o\ti)^{N-1},
 \]
which immediately implies the lower bound on $\ell_{t,i} - \EEt{\hloss_{t,i}}$.
For proving the upper bound, observe that
\begin{align*}
&\loss\ti(1-o\ti)^{N-1}\leq(1-r_t)^{N-1} \leq e^{-r_t (N-1)} \leq\frac{1}{\sqrt{T}}
\end{align*}
holds by our assumption on $r_t$, where we used the elementary inequality $1-x\le e^{-x}$ that holds for all $x\in\mathbb{R}$.
\end{proof}

The next theorem states our main result concerning \alg with an adaptive learning rate.
\begin{theorem}
\label{thm:alg}
Assume that $r_t\ge \frac{\log T}{2N-2}$ holds for all $t$ and set
\[
 \eta_{t} = \sqrt{\frac{\log N }{N^2+\sum_{s = 1}^{t-1} \sumi p_{s,i}(\hloss_{s,i})^2}}.
\]
Then, the expected regret of \alg satisfies
\[
R_T\leq 2\sqrt{\pa{N^2 +  \sumT\frac{1}{r_t}}\log N} + \sqrt{T}.
\]
\end{theorem}

\section{PROOF OF THEOREM \ref{thm:alg}}

In this section, we present details of the proof of Theorem~\ref{thm:alg} but first, we state an auxiliary lemma.

\begin{lemma}[Lemma~3.5 of \citealp{auer2002adaptive}]
\label{lem:lem1}
Let $b_1$, $b_2$, \dots, $b_T$ be non-negative real numbers. Then 
\[
\sumT\frac{b_t}{\sqrt{\sum_{s=1}^t b_s}} \leq 2\sqrt{\textstyle\sumT b_t}.
\]
\end{lemma}

\begin{proof}
The proof is based on the inequality $x/2 \leq 1-\sqrt{1-x}$ for $x\leq1$. Setting $x = b_t/\sum_{s=1}^t b_s$ and 
multiplying both sides of the inequality by $\sqrt{\sum_{s=1}^t b_s}$ we get
\[
\frac{b_t}{\sqrt{\sum_{s = 1}^t b_t}} \leq \sqrt{\textstyle\sum_{s=1}^t b_s}-\sqrt{\textstyle\sum_{s=1}^t b_s-b_t}.
\]
The proof is concluded by summing over $t$.
\end{proof}
%

The first part of the analysis 
follows the proof of Lemma~1 by \citet{gyorfi2007sequential}. Defining $W'_{t} = \frac 1N\sumi e^{-\eta_{t-1}\hL_{t-1,i}}$, we get
\begin{align}
&\nonumber\frac{1}{\eta_t}\log\frac{W'_{t+1}}{W_t}=\frac{1}{\eta_t}\log\sumi\frac{\frac 1Ne^{-\eta_t\hL_{t-1,i}} e^{-\eta_t\hloss\ti}}{W_t} \\
&\nonumber\quad =\frac{1}{\eta_t}\log\sumi p\ti e^{-\eta_t\hloss\ti}	\\\label{eq:tailor}
&\quad \leq
\frac{1}{\eta_t}\log\sumi p\ti\left(1-\eta_t\hloss\ti+
(\eta_t\hloss\ti)^2\right)	\\
&\nonumber\quad =\frac{1}{\eta_t}\log\left(1-\eta_t\sumi p\ti\hloss\ti
+\eta_t^2\sumi p\ti(\hloss\ti)^2\right), \
\end{align}
where in (\ref{eq:tailor}), we used the inequality $\exp(-x) \le 1 - x + x^2$ that holds for 
$x\ge -1$. Further, we used the inequality $\log(1-x)\leq-x$, 
which holds for all $x\leq 1$, to upper bound the last term.

Using $\eta_{t+1}\leq \eta_t$ and Jensen's inequality, we get
\begin{align*}
W_{t+1} &= \sumi\frac{1}{N}e^{-\eta_{t+1}\hL_{t,i}} =
\sumi\frac{1}{N}\left(e^{-\eta_{t}\hL_{t,i}}\right)^{\frac{\eta_{t+1}}{
\eta_t}}\\
&\leq\left(\sumi\frac{1}{N}e^{
-\eta_{t}\hL_{t,i}}\right)^{\frac{\eta_{t+1}}{\eta_t}} =
(W'_{t+1})^{\frac{\eta_{t+1}}{\eta_t}},
\end{align*}
which, together with the last inequality, gives us 
\[
\sumi{p}\ti\hloss\ti\leq\frac{{\eta}_t}{2}\sumi {p}\ti\pa{\hloss\ti}^2 + 
\left[\frac{\log {W}_t}{{\eta}_t}-\frac{\log {W}_{t+1}}{{\eta}_{t+1}}\right]
\]
for every $t\in[T]$. Taking expectations and summing over time, we get
\begin{align*}
\EE{\sumT\sumi{p}\ti\hloss\ti}&\leq \EE{\sumT\frac{{\eta}_t}{2}\sumi {p}\ti\pa{\hloss\ti}^2}	\\
& + \EE{\sumT\left(\frac{\log {W}_t}{{\eta}_t}-\frac{\log 
{W}_{t+1}}{{\eta}_{t+1}}\right)}.
\end{align*}

The goal of the second part of the analysis is to construct bounds for each of the three expectations in the previous inequality.
For the term on the left-hand side, we use Lemma~\ref{lemma:lbias} to get the lower-bound
\[
\EE{\sumT\sumi{p}\ti\hloss\ti} \geq \sumT\sumi p\ti\ell\ti + \sqrt{T}.
\]

Note that this is the only step in the analysis where the actual magnitude (and not just the sign) of the bias of the loss estimates shows up. 
Anything bigger than $\sqrt{T}$ would degrade our final regret bound.

We are left with bounding the two terms on the right-hand side.
To simplify some notation below, let us define $b_t = \sumi p\ti(\hloss\ti)^2$.
By our definition of $\etat$ and the help of Lemma~\ref{lem:lem1}, we can bound the first term on the right hand side as
\begin{align*}
\EE{\sumT\frac{{\eta}_t b_t}{2}} & = \EE{\sumT\frac{b_t\sqrt{\log N}}{2\sqrt{N^2+\sum_{s = 1}^{t-1}b_s}}}\\
&\leq \EE{\sqrt{\pa{N^2 + \textstyle\sumT b_t}\log N}}	
\\
&\leq \sqrt{\pa{N^2 + \textstyle\sumT \EE{b_t}}\log N},
\end{align*}
where we also used the fact that $N^2 \ge b_t$ and Jensen's inequality in the last line. We continue by bounding $\EE{b_t}$:
\begin{equation}\label{eq:l2bound}
\begin{split}
 \EEt{\sum_{i=1}^N p_{t,i} (\hloss_{t,i})^2} &= \sumi p\ti \loss\ti^2 \EEt{O\ti G\ti^2} 
 \\
 &\le \sumi p\ti o_{t,i} \frac{2-o\ti}{o\ti^2}\le \frac {2}{r_t},
\end{split}
\end{equation}
where we used $o\ti\ge r_t$ together with the second part of Lemma \ref{lem:expectationsOfG} which gives us

\begin{align*}
\EEt{G_{t,i}^2} &=  \frac{2-o\ti}{o\ti^2} + \frac{1}{o\ti^2}(1-o\ti)^{N-2}\times\\
&\quad\times\Big(o\ti^2+o\ti-2+2o\ti(N-2)(o-1)\Big)\\
&\leq\frac{2-o\ti}{o\ti^2},
\end{align*}

since both $o\ti^2+o\ti-2$ and $2o\ti(N-2)(o-1)$ are non-positive. Thus, we obtain
\begin{equation}\label{eq:secterm}
 \EE{\sumT\frac{{\eta}_t b_t}{2}} \le 
 \sqrt{\pa{\sumT\frac{1}{r_t} + N^2} \log N}.
\end{equation}

Finally, using $W_1 = 1$, the sum in the last expectation telescopes to 
\begin{align*}
\EE{\sumT\left(\frac{\log {W}_t}{{\eta}_t}-\frac{\log {W}_{t+1}}{{\eta}_{t+1}}\right)} = 
\EE{-\frac{\log {W}_{T+1}}{{\eta}_{T+1}}}.
\end{align*}
Using the definition of $W_t$, we get that
\begin{align*}
\EE{-\frac{\log {W}_{T+1}}{{\eta}_{T+1}}} &\leq \EE{-\frac{\log {w}_{T+1,j}}{{\eta}_{T+1}}} 	\\
&\leq \EE{\frac{\log N}{\eta_{T+1}}}+\EE{\hL_{T,j}}
\end{align*}
holds for any arm $j\in [N]$.
Now note that the first term can be bounded by using the definition of $\eta_{T+1}$ with the help of ~\eqref{eq:l2bound} and Jensen's inequality.
Using $\EEt{\hloss_{t,i}} \le \loss_{t,i}$ from Lemma~\ref{lemma:lbias}
and combining everything together, we obtain the regret bound
\begin{align*}
R_{T} &= \EE{\sumT p_{t,i} \loss_{t,i}}-\min_{j\in [N]}\EE{\sumTk\loss_{t,j}}
\\
&\leq 2\sqrt{\left(N^2 + \sumT\frac{1}{r_t}\right)\log N} + \sqrt{T}.
\end{align*}

\section{EXPERIMENTS}
\label{sec:experiments}

\begin{figure*}[t]
	\centering
	\begin{subfigure}[t]{0.32\textwidth}
		\centering
		\includegraphics[width = 1\textwidth]{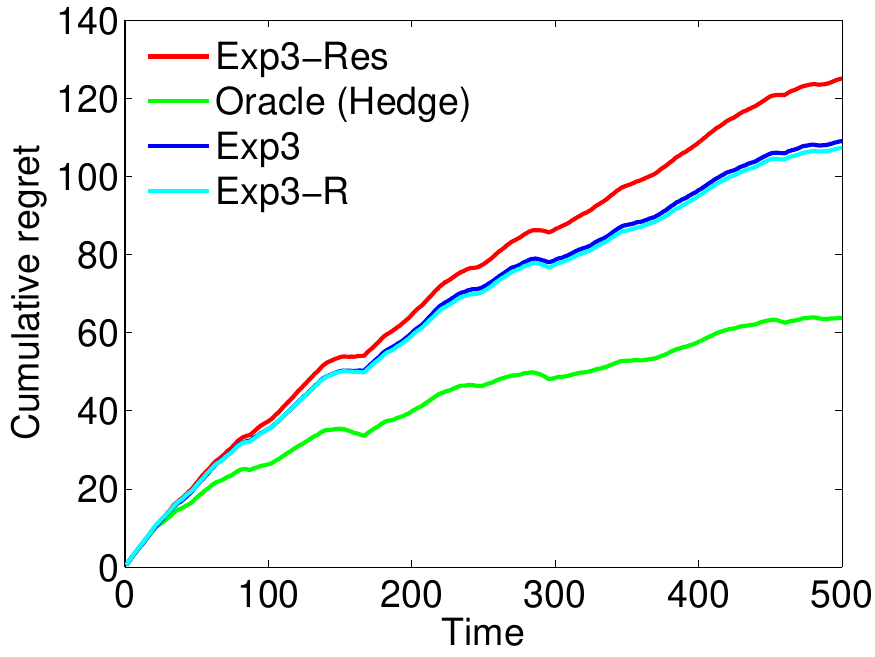}
		\caption{Static sequence $(r_t)_t^T$, $r_t = 0$}
	\end{subfigure}	
	\
	\begin{subfigure}[t]{0.32\textwidth}
		\centering
		\includegraphics[width = 1\textwidth]{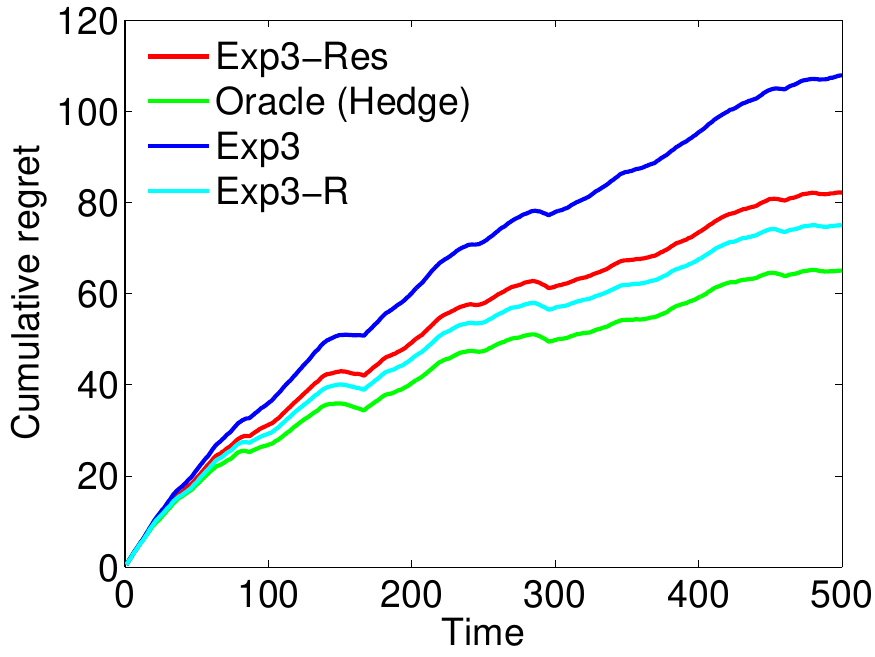}
		\caption{Static sequence $(r_t)_t^T$, $r_t = 0.06\approx \log(T)/(2N-2)$}
	\end{subfigure}
	\
	\begin{subfigure}[t]{0.32\textwidth}
		\centering
		\includegraphics[width = 1\textwidth]{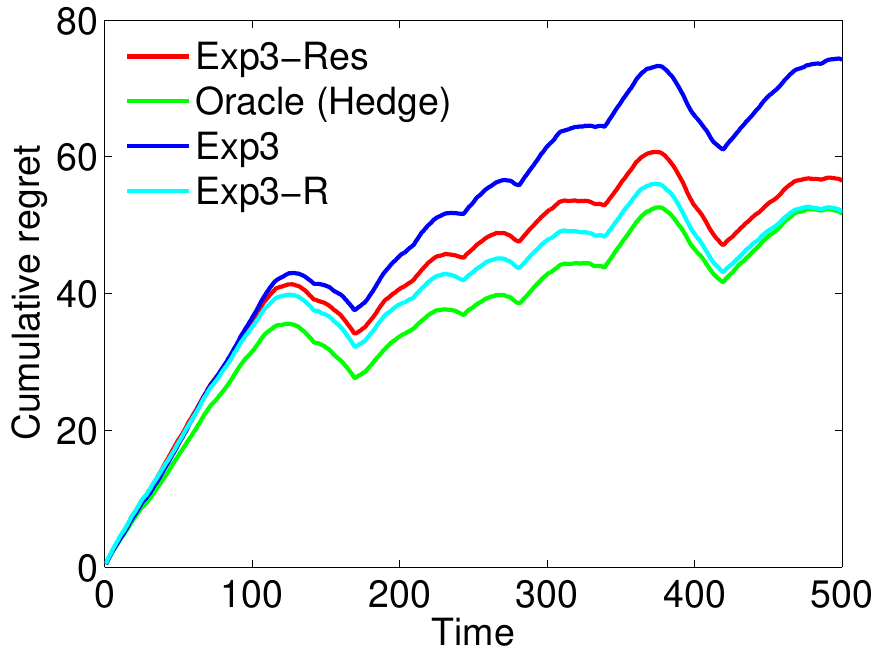}
		\caption{Changing sequence $(r_t)_t^T$ with uniformly distributed $r_t$ on $[0,0.2]$}
	\end{subfigure}
	\vskip 2em
	\centering
	\begin{subfigure}[t]{0.32\textwidth}
		\centering
		\includegraphics[width = 1\textwidth]{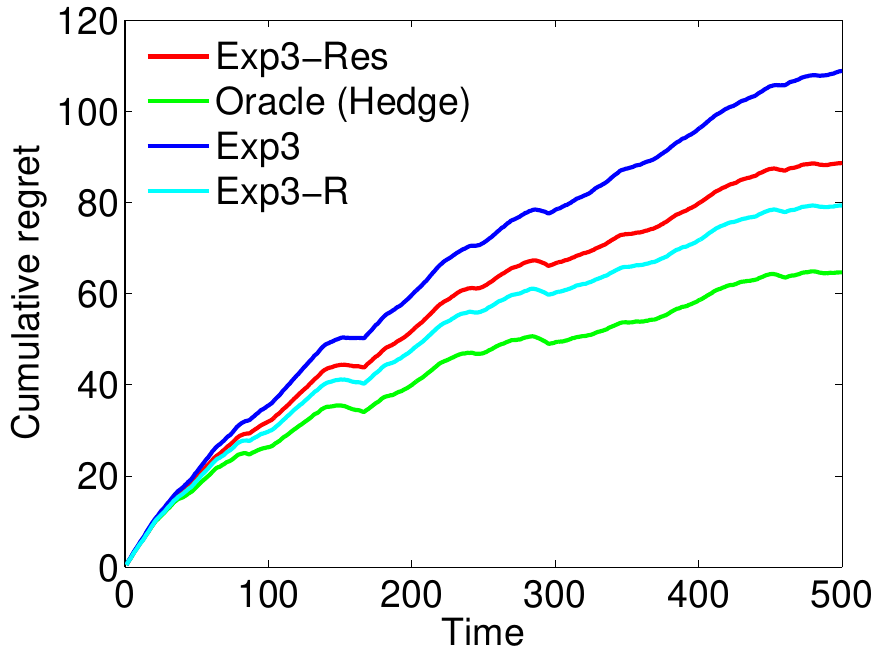}
		\caption{Sequence $(r_t)_t^T$ generated as a random walk on $[0,0.1]$}
	\end{subfigure}		
	\
	\begin{subfigure}[t]{0.32\textwidth}
		\centering
		\includegraphics[width = 1\textwidth]{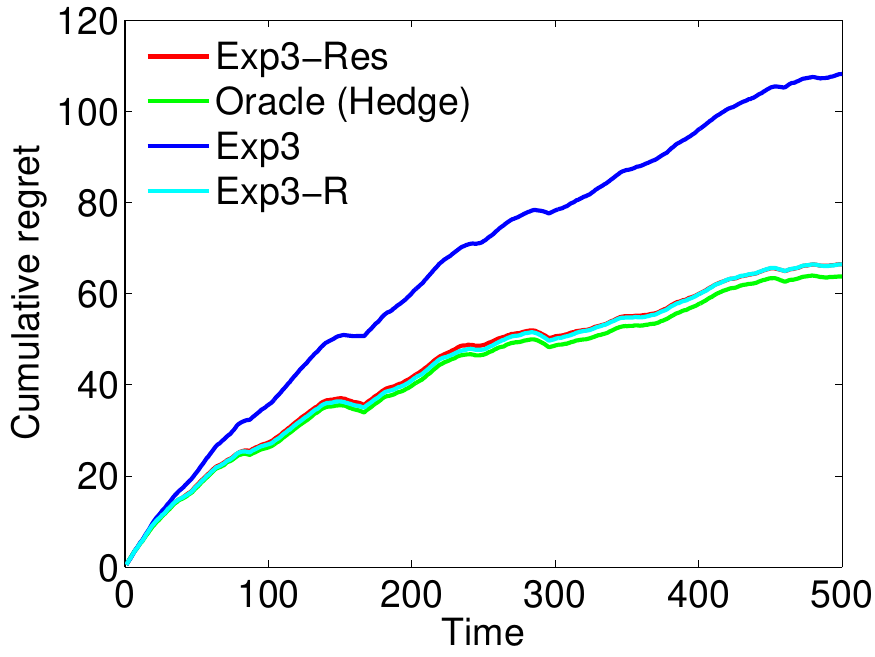}
		\caption{Sequence $(r_t)_t^T$ generated as a random walk on $[0,1]$}
	\end{subfigure}
	\
	\begin{subfigure}[t]{0.32\textwidth}
		\centering
		\includegraphics[width = 1\textwidth]{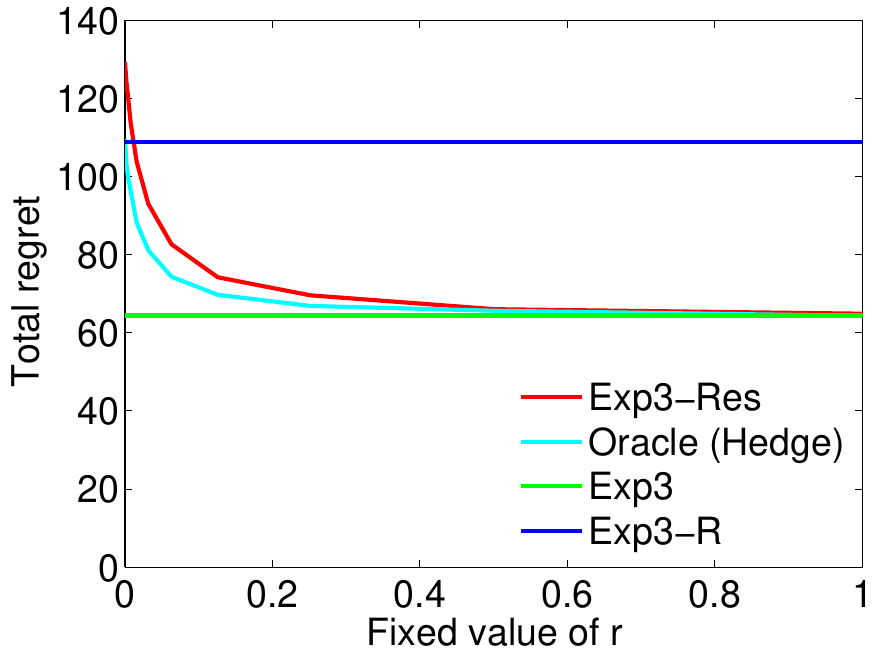}
		\caption{Total regret for different values of static $(r_t)_t^T$}
	\end{subfigure}
	\vskip 1em
	\caption{Comparison of algorithms for different amount of side information sequences (different sequences $(r_t)_t^T$)}
\end{figure*}


In this section, we study the empirical performance of \alg compared to three other algorithms:

\begin{itemize}
\item \texttt{Exp3} -- a basic adversarial multi-armed bandit algorithm which uses only loss observations of chosen arms and discards all 
side observations.
\item \exporacle\xspace -- full-information algorithm with access to losses of every action in every time step, regardless of the value of 
$r_t$. Our particular choice is \hedge \citep{littlestone1994weighted,freund1997decision}.
\item \expr\ -- a variant of the \alg algorithm with access to the sequence $(r_t)_t^T$, using (\ref{eq:optest}) to construct unbiased loss estimate instead of using geometric resampling.
\end{itemize}

The most interesting parameter of our experiment is the sequence $\pa{r_t}$, since it controls amount of 
side observation presented to the learner. In order to show that \alg can effectively make use of  the additional information provided by 
the environment, we designed several sequences $\pa{r_t}$ with different amounts of side observation provided to the learner. In the 
case of small $r_t$-s, the problem is almost as difficult as the multi-armed bandit problem. On the 
other hand, in the case of large $r_t$-s, the problem is almost as easy as the full-information problem.
Therefore, we expect that the performance of \alg will interpolate between the performance of the \expr and \exporacle 
algorithms depending on the values of the $r_t$-s. In the next section, we validate this claim empirically.

\subsection{EXPERIMENT DETAILS}
To ensure sufficient challenge for the algorithms, we have generated a sequence of losses as a random walk 
for each arm with independent increments uniformly distributed on $[-0.1,0.1]$ while enforcing the random walks to stay within $[0,1]$ by setting the value of a random walk to 0 or 1, respectively, if the random walk gets outside the boundaries. The loss sequence is 
fixed through all of the experiments to demonstrate the impact of the sequence $(r_t)_t^T$ on the regret of algorithms. We have 
observed qualitatively similar behavior for other loss sequences. 

We fix the number of arms in all of the experiments as $50$, and the time horizon as $500$. Every curve represents an average of 
100 runs.

\subsection{RESULT OF THE EXPERIMENTS}
We performed experiments on many different loss sequences and sequences of $r_t$-s. 
Since the results are essentially the same for all the different sequences, we included in the present paper just the results for one loss sequence with 
different sequences of $r_t$-s.
In the case of $r_t\geq\log(T)/(2N-2)$, the case of high probability of having some side observation, the performance of the algorithm \alg 
proposed in the present paper is comparable to the performance of the idealistic \expr which knows exact value of $r_t$ in every time step. 
Moreover, if the average $r_t$ is close to 1, the performance of the proposed algorithm is close to the performance of \exporacle which 
observes all the losses. If the average $r_t$ is close to zero, the performance of the algorithm is a little bit worse than the performance of basic 
\exph. This is also supported by the theory, since our algorithm is not able to construct reliable estimates in the case of small $r_t$-s.

\section{CONCLUSION \& FUTURE WORK}\label{sec:conc}

In this paper, we considered multi-armed bandit problems with stochastic side observations modeled by Erd\H os--R\'enyi graphs. Our 
contribution is a computationally efficient algorithm that operates under the assumption $r_t \ge \log T/(2N-2)$, which essentially 
guarantees that at least one piece of side observation is generated in every round, with high probability. In this case, our algorithm 
guarantees a regret bound of $\cO\pa{\sqrt{\log N\sumT \frac{1}{r_t} }}$ (Theorem~\ref{thm:alg}). In this section, we discuss several open 
questions regarding this result.


The most obvious question is whether it is possible to remove our assumptions on the values of $r_t$. We can only give a definite answer in 
the simple case when all $r_t$'s are identical: In this case, one can think of simply computing the empirical frequency $\widehat{r}_t$ of all 
previous side observations in round $t$ to estimate the constant $r$, plug the result into~\eqref{eq:optest}, and then use the 
resulting loss estimates in an exponential-weighting scheme. It is relatively straightforward (but also rather tedious) to show that the 
resulting algorithm satisfies a regret bound of $\wt{\cO}\pa{\sqrt{T/r}}$ for all possible values of $r$, thanks to the fact that $\hat{r}_t$ 
quickly concentrates around the true value of $r$. Notice however that this approach clearly breaks down if the $r_t$'s change over time.

In the case of changing $r_t$'s, the number of observations we can use to estimate $r_t$ is severely limited, so much that we cannot expect 
any direct estimate of $r_t$ to concentrate around the true value. Our algorithm proposed in Section~\ref{sec:alg} gets around this problem 
by directly estimating the importance weights $1/o_{t,i}$ instead of $r_t$, which enables us to construct reliable loss estimates, 
although only at the price of our assumption on the range of $r_t$. While we acknowledge that this assumption can be difficult to confirm a 
priori in practice, we remark that we find it quite surprising that \emph{any algorithm whatsoever} can take advantage of 
such limited observations, even under such a restriction. We also point out that for values of $r_t$ that are consistently below our bound, 
it is not possible to substantially improve the regret bounds of \exph which are of $\wt{\cO}\pa{\sqrt{TN}}$, as shown by the lower bounds of 
\citet{alon2013from}. We expect that in several practical applications, one can verify whether the $r_t$'s satisfy our assumption or not, 
and decide to use \alg or \exph accordingly. In fact, our experiments suggest that our algorithm performs well even if neither of these two 
assumptions are verified: we have seen that the empirical performance of \alg is only slightly worse than that of \exph even when 
the values of $r_t$ are very small (Section \ref{sec:experiments}).
Still, finding out whether our restriction on $r_t$ can be relaxed in general 
is a very important and interesting question left for future study.

An important corollary of our results is that, under some assumptions, it is possible to leverage side observations in a non-trivial 
way without having access to the second neighborhoods in the side-observation graphs as defined by \citet{mannor2011from}. This complements 
the recent results of \citet{CHK16}, who show that non-stochastic side-observations may provide non-trivial advantage over bandit feedback 
when the losses are stochastic even when the side-observation graphs are unobserved, but learning with unobserved feedback graphs can be 
as hard as learning with bandit feedback when both the losses and the graphs are generated by an adversary. A natural question that our 
work leads to is whether it is possible to efficiently leverage side-observations under significantly weaker assumptions on the observation 
model.

\paragraph{Acknowledgements}
\label{sec:Acknowledgements}
The research presented in this paper was supported by CPER Nord-Pas de Calais/FEDER DATA Advanced data science and technologies 2015-2020, French Ministry of
Higher Education and Research, Nord-Pas-de-Calais Regional Council,  French National Research Agency project ExTra-Learn (n.ANR-14-CE24-0010-01), 
and by UPFellows Fellowship (Marie Curie COFUND program n${^\circ}$ 600387).

\bibliography{library}

\begin{thebibliography}{}

\bibitem[Allenberg et~al., 2006]{allenberg2006hannan}
Allenberg, C., Auer, P., Gy{\"{o}}rfi, L., and Ottucs{\'{a}}k, {\relax Gy}.
  (2006).
\newblock {Hannan consistency in on-line learning in case of unbounded losses
  under partial monitoring}.
\newblock In {\em Algorithmic Learning Theory (ALT)}, pages 229--243.

\bibitem[Alon et~al., 2015]{alon2015online}
Alon, N., Cesa-Bianchi, N., Dekel, O., and Koren, T. (2015).
\newblock {Online learning with feedback graphs: Beyond bandits}.
\newblock In {\em Conference on Learning Theory (COLT)}.

\bibitem[Alon et~al., 2013]{alon2013from}
Alon, N., Cesa-Bianchi, N., Gentile, C., and Mansour, Y. (2013).
\newblock {From bandits to experts: A tale of domination and independence}.
\newblock In {\em Neural Information Processing Systems (NeurIPS)}.

\bibitem[Audibert and Bubeck, 2010]{audibert2010regret}
Audibert, J.-Y. and Bubeck, S. (2010).
\newblock {Regret bounds and minimax policies under partial monitoring}.
\newblock {\em Journal of Machine Learning Research}, 11:2785--2836.

\bibitem[Auer et~al., 2002a]{auer2002nonstochastic}
Auer, P., Cesa-Bianchi, N., Freund, Y., and Schapire, R.~E. (2002a).
\newblock {The non-stochastic multi-armed bandit problem}.
\newblock {\em SIAM Journal on Computing}, 32(1):48--77.

\bibitem[Auer et~al., 2002b]{auer2002adaptive}
Auer, P., Cesa-Bianchi, N., and Gentile, C. (2002b).
\newblock {Adaptive and self-confident on-line learning algorithms}.
\newblock {\em Journal of Computer and System Sciences}, 64:48--75.

\bibitem[Buccapatnam et~al., 2014]{buccapatnam2014stochastic}
Buccapatnam, S., Eryilmaz, A., and Shroff, N.~B. (2014).
\newblock {Stochastic bandits with side observations on networks}.
\newblock In {\em International Conference on Measurement and Modeling of
  Computer Systems}.

\bibitem[Caron et~al., 2012]{caron2012leveraging}
Caron, S., Kveton, B., Lelarge, M., and Bhagat, S. (2012).
\newblock {Leveraging side observations in stochastic bandits.}
\newblock In {\em Conference on Uncertainty in Artificial Intelligence (UAI)}.

\bibitem[Carpentier and Valko, 2016]{carpentier2016revealing}
Carpentier, A. and Valko, M. (2016).
\newblock {Revealing graph bandits for maximizing local influence}.
\newblock In {\em International Conference on Artificial Intelligence and Statistics (AISTATS)}, pages 10--18.

\bibitem[Cesa-Bianchi and Lugosi, 2006]{cesa-bianchi2006prediction}
Cesa-Bianchi, N. and Lugosi, G. (2006).
\newblock {\em {Prediction, learning, and games}}.
\newblock Cambridge University Press, New York, NY.

\bibitem[Cesa-Bianchi et~al., 2005]{cesa-bianchi2005minimizing}
Cesa-Bianchi, N., Lugosi, G., and Stoltz, G. (2005).
\newblock {Minimizing regret with label efficient prediction}.
\newblock {\em IEEE Transactions on Information Theory}, 51(6):2152--2162.

\bibitem[Cohen et~al., 2016]{CHK16}
Cohen, A., Hazan, T., and Koren, T. (2016).
\newblock Online learning with feedback graphs without the graphs.
\newblock In {\em International Conference on Machine Learning (ICML)}.

\bibitem[Freund and Schapire, 1997]{freund1997decision}
Freund, Y. and Schapire, R.~E. (1997).
\newblock {A decision-theoretic generalization of on-line learning and an
  application to boosting}.
\newblock {\em Journal of Computer and System Sciences}, 55:119--139.

\bibitem[{Gy}{\"{o}}rfi and Ottucs{\'{a}}k, 2007]{gyorfi2007sequential}
{Gy}{\"{o}}rfi, L. and Ottucs{\'{a}}k, {\relax Gy}. (2007).
\newblock {Sequential prediction of unbounded stationary time series}.
\newblock {\em IEEE Transactions on Information Theory}, 53(5):1866--1872.

\bibitem[Koc{\'{a}}k et~al., 2016]{kocak2016online}
Koc{\'{a}}k, T., Neu, G., and Valko, M. (2016).
\newblock {Online learning with noisy side observations}.
\newblock In {\em International Conference on Artificial Intelligence and Statistics (AISTATS)}, pages 1186--1194.

\bibitem[Koc{\'{a}}k et~al., 2014]{kocak2014efficient}
Koc{\'{a}}k, T., Neu, G., Valko, M., and Munos, R. (2014).
\newblock {Efficient learning by implicit exploration in bandit problems with
  side observations}.
\newblock In {\em Neural Information Processing Systems (NeurIPS)}, pages 613--621.

\bibitem[Littlestone and Warmuth, 1994]{littlestone1994weighted}
Littlestone, N. and Warmuth, M. (1994).
\newblock {The weighted majority algorithm}.
\newblock {\em Information and Computation}, 108(2):212--261.

\bibitem[Mannor and Shamir, 2011]{mannor2011from}
Mannor, S. and Shamir, O. (2011).
\newblock {From bandits to experts: On the value of side-observations}.
\newblock In {\em Neural Information Processing Systems (NeurIPS)}.

\bibitem[Neu and Bart{\'{o}}k, 2013]{neu2013efficient}
Neu, G. and Bart{\'{o}}k, G. (2013).
\newblock {An efficient algorithm for learning with semi-bandit feedback}.
\newblock In {\em Algorithmic Learning Theory (ALT)}.

\bibitem[Seldin et~al., 2014]{seldin2014prediction}
Seldin, Y., Bartlett, P., Crammer, K., and Abbasi-Yadkori, Y. (2014).
\newblock {Prediction with limited advice and multiarmed bandits with paid
  observations}.
\newblock In {\em International Conference on Machine Learning (ICML)}.

\end{thebibliography}
\bibliographystyle{apalike}

\end{document}